\documentclass[11pt,a4paper]{article}
\usepackage[nohyperref]{emnlp2017}
\usepackage{times}
\usepackage{latexsym}
\usepackage{url}
\usepackage[pdftex]{graphicx}
\usepackage{algorithm}
\usepackage{algorithmic}
\usepackage{amsmath}
\usepackage{amssymb}
\usepackage{amsthm}

\emnlpfinalcopy
\def\emnlppaperid{***}

\newcommand{\tvec}[1]{{\mathbf t}_{#1}}
\newcommand{\cvec}[1]{{\mathbf c}_{#1}}
\newcommand{\vocab}{\mathcal{W}}
\newcommand{\vocabsize}{|\vocab|}
\newcommand{\var}[1]{\rho_{#1}}
\newcommand{\rv}[2]{\mathrm{#1}_{#2}}

\newtheorem{theorem}{Theorem}
\newtheorem{lemma}[theorem]{Lemma}

\theoremstyle{definition}
\newtheorem{definition}{Definition}

\title{Incremental Skip-gram Model with Negative Sampling}

\author{Nobuhiro Kaji \and Hayato Kobayashi \\
Yahoo Japan Corporation\\
{\tt \{nkaji,hakobaya\}@yahoo-corp.jp}}
\date{}

\begin{document}
\maketitle
\begin{abstract}
This paper explores an incremental training strategy for the skip-gram model with negative sampling (SGNS) from both empirical and theoretical perspectives. Existing methods of neural word embeddings, including SGNS, are multi-pass algorithms and thus cannot perform incremental model update. To address this problem, we present a simple incremental extension of SGNS and provide a thorough theoretical analysis to demonstrate its validity. Empirical experiments demonstrated the correctness of the theoretical analysis as well as the practical usefulness of the incremental algorithm.
\end{abstract}

\section{Introduction} \label{sec:intro}

Existing methods of neural word embeddings are typically designed to go through the entire training data multiple times. For example, negative sampling \cite{Mikolov13c} needs to pre-compute the noise distribution from the entire training data before performing Stochastic Gradient Descent (SGD). It thus needs to go through the training data at least twice. Similarly, hierarchical soft-max \cite{Mikolov13c} has to determine the tree structure and GloVe \cite{Pennington14} has to count co-occurrence frequencies before performing SGD.

The fact that those existing methods are multi-pass algorithms means that they cannot perform incremental model update when additional training data is provided. Instead, they have to re-train the model on the old and new training data from scratch. 

However, the re-training is obviously inefficient since it has to process the entire training data received thus far whenever new training data is provided. This is especially problematic when the amount of the new training data is relatively smaller than the old one. One such a situation is that the embedding model is updated on a small amount of training data that includes newly emerged words for instantly adding them to the vocabulary set. Another situation is that the word embeddings are learned from ever-evolving data such as news articles and microbologs \cite{Peng17} and the embedding model is periodically updated on newly generated data (\textit{e.g.}, once in a week or month).

This paper investigates an incremental training method of word embeddings with a focus on the skip-gram model with negative sampling (SGNS) \cite{Mikolov13c} for its popularity. We present a simple incremental extension of SGNS, referred to as \textit{incremental SGNS}, and provide a thorough theoretical analysis to demonstrate its validity. Our analysis reveals that, under a mild assumption, the optimal solution of incremental SGNS agrees with the original SGNS when the training data size is infinitely large. See Section~\ref{sec:analysis} for the formal and strict statement. Additionally, we present techniques for the efficient implementation of incremental SGNS.

Three experiments were conducted to assess the correctness of the theoretical analysis as well as the practical usefulness of incremental SGNS\@. The first experiment empirically investigates the validity of the theoretical analysis result. The second experiment compares the word embeddings learned by incremental SGNS and the original SGNS across five benchmark datasets, and demonstrates that those word embeddings are of comparable quality. The last experiment explores the training time of incremental SGNS, demonstrating that it is able to save much training time by avoiding expensive re-training when additional training data is provided.

\section{SGNS Overview} \label{sec:sgns}

As a preliminary, this section provides a brief overview of SGNS\@.

Given a word sequence, $w_1,w_2,\dots,w_{n}$, for training, the skip-gram model seeks to minimize the following objective to learn word embeddings:
\begin{align*}
\mathcal{L}_\text{SG}=-\frac{1}{n}\sum_{i=1}^{n}\sum_{\substack{\mid j\mid\leq c\\j\not=0}}\log p(w_{i+j}\mid w_{i}),
\end{align*}
where $w_{i}$ is a target word and $w_{i+j}$ is a context word within a window of size $c$. $p(w_{i+j}\mid w_{i})$ represents the probability that $w_{i+j}$ appears within the neighbor of $w_i$, and is defined as
\begin{align}
p(w_{i+j}\mid w_{i})=\frac{\exp(\tvec{w_i}\cdot\cvec{w_{i+j}})}{\sum_{w\in\vocab}\exp(\tvec{w_i}\cdot\cvec{w})},\label{eq:softmax}
\end{align}
where $\tvec{w}$ and $\cvec{w}$ are $w$'s embeddings when it behaves as a target and context, respectively. $\vocab$ represents the vocabulary set.

Since it is too expensive to optimize the above objective, Mikolov et al.~\shortcite{Mikolov13c} proposed negative sampling to speed up skip-gram training. This approximates Eq.~(\ref{eq:softmax}) using sigmoid functions and $k$ randomly-sampled words, called \textit{negative samples}. The resulting objective is given as
\begin{align*}
\mathcal{L}_\text{SGNS}\!=\!-\frac{1}{n}\sum_{i=1}^{n}\sum_{\substack{\mid j\mid\leq c\\j\not=0}}\psi_{w_i,w_{i+j}}^{+}\!+\!k\mathbb{E}_{v\sim q(v)}[\psi_{w_i,v}^{-}],
\end{align*}
where $\psi_{w,v}^{+}=\log\sigma(\tvec{w}\cdot\cvec{v})$, $\psi_{w,v}^{-}=\log\sigma(-\tvec{w}\cdot\cvec{v})$, and $\sigma(x)$ is the sigmoid function. The negative sample $v$ is drawn from a smoothed unigram probability distribution referred to as \textit{noise distribution}: $q(v)\propto f(v)^{\alpha}$, where $f(v)$ represents the frequency of a word $v$ in the training data and $\alpha$ is a smoothing parameter ($0<\alpha\leq 1$). 

The objective is optimized by SGD\@. Given a target-context word pair ($w_{i}$ and $w_{i+j}$) and $k$ negative samples ($v_{1},v_{2},\dots,v_{k}$) drawn from the noise distribution, the gradient of
$-\psi_{w_i,w_{i+j}}^{+}-k\mathbb{E}_{v\sim q(v)}[\psi_{w_i,v}^{-}]\approx-\psi_{w_i,w_{i+j}}^{+}-\sum_{k'=1}^{k}\psi_{w_i,v_{k'}}^{-}$
is computed. Then, the gradient descent is performed to update $\tvec{w_i}$, $\cvec{w_{i+j}}$, and $\cvec{v_1},\dots,\cvec{v_k}$.

SGNS training needs to scan the entire training data multiple times because it has to pre-compute the noise distribution $q(v)$ before performing SGD\@. This makes it difficult to perform incremental model update when additional training data is provided.

\section{Incremental SGNS} \label{sec:isgns}

This section explores incremental training of SGNS\@. The incremental training algorithm (Section \ref{subsec:algorithm}), its efficient implementation (Section \ref{subsec:implementation}), and the computational complexity (Section \ref{subsec:isgns_complexity}) are discussed in turn.

\subsection{Algorithm} \label{subsec:algorithm}

Algorithm~\ref{alg:isgns} presents \textit{incremental SGNS}, which goes through the training data in a single-pass to update word embeddings incrementally. Unlike the original SGNS, it does not pre-compute the noise distribution. Instead, it reads the training data word by word%
\footnote{In practice, Algorithm~\ref{alg:isgns} buffers a sequence of words $w_{i-c},\dots,w_{i+c}$ (rather than a single word $w_i$) at each step, as it requires an access to the context words $w_{i+j}$ in line 7. This is not a practical problem because the window size $c$ is usually small and independent from the training data size $n$.} to incrementally update the word frequency distribution and the noise distribution while performing SGD\@. Hereafter, the original SGNS (\textit{c.f.}, Section~\ref{sec:sgns}) is referred to as \textit{batch SGNS} to emphasize that the noise distribution is computed in a batch fashion.

The learning rate for SGD is adjusted by using AdaGrad \cite{Duchi11}. Although the linear decay function has widely been used for training batch SGNS \cite{Mikolov}, adaptive methods such as AdaGrad are more suitable for the incremental training since the amount of training data is unknown in advance or can increase unboundedly.

It is straightforward to extend the incremental SGNS to the mini-batch setting by reading a subset of the training data (or mini-batch), rather than a single word, at a time to update the noise distribution and perform SGD (Algorithm~\ref{alg:minibatch}). Although this paper primarily focuses on the incremental SGNS, the mini-batch algorithm is also important in practical terms because it is easier to be multi-threaded.

Alternatives to Algorithms~\ref{alg:minibatch} might be possible. Other possible approaches include computing the noise distribution separately on each subset of the training data, fixing the noise distribution after computing it from the first (possibly large) subset, and so on. We exclude such alternatives from our investigation because it is considered difficult to provide them with theoretical justification. 

\begin{algorithm}[t]
\footnotesize
\caption{Incremental SGNS}
\label{alg:isgns}
\begin{algorithmic}[1]
\STATE $f(w)\leftarrow 0$ for all $w\in\vocab$
\FOR{$i=1,\dots,n$}
\STATE $f(w_i)\leftarrow f(w_i)+1$
\STATE $q(w)\leftarrow\frac{f(w)^{\alpha}}{\sum_{w'\in\vocab}f(w')^{\alpha}}$ for all $w\in\vocab$
\FOR{$j=-c,\dots,-1,1,\dots,c$}
\STATE draw $k$ negative samples from $q(w)$: $v_1,\dots,v_{k}$
\STATE use SGD to update $\tvec{w_i}$, $\cvec{w_{i+j}}$, and $\cvec{v_1},\dots,\cvec{v_k}$
\ENDFOR
\ENDFOR
\end{algorithmic}

\end{algorithm}
\begin{algorithm}[t]
\footnotesize
\caption{Mini-batch SGNS}
\label{alg:minibatch}
\begin{algorithmic}[1]
\FOR{each subset $\mathcal{D}$ of the training data}
\STATE update the noise distribution using $\mathcal{D}$
\STATE perform SGD over $\mathcal{D}$
\ENDFOR
\end{algorithmic}
\end{algorithm}

\subsection{Efficient implementation} \label{subsec:implementation}

Although the incremental SGNS is conceptually simple, implementation issues are involved. 

\subsubsection{Dynamic vocabulary} \label{subsec:dynamic_vocab}

One problem that arises when training incremental SGNS is how to maintain the vocabulary set. Since new words emerge endlessly in the training data, the vocabulary set can grow unboundedly and exhaust a memory.

We address this problem by dynamically changing the vocabulary set. The Misra-Gries algorithm \cite{Misra82} is used to approximately keep track of top-$m$ frequent words during training, and those words are used as the dynamic vocabulary set. This method allows the maximum vocabulary size to be explicitly limited to $m$, while being able to dynamically change the vocabulary set.

\begin{table}[t]
\footnotesize
\begin{center}
\begin{tabular}{@{}c@{\;}c@{}}
 \begin{tabular}{c|ccc}
  $w$ & \textit{a} & \textit{b} & \textit{c} \\\hline
  $q(w)$ & $0.5$ & $0.3$ & $0.2$ \\
  \end{tabular}
&
\begin{tabular}{c}
  $T=(a,a,a,a,a,b,b,b,c,c)$
\end{tabular}
\\
\end{tabular}
\caption{Example noise distribution $q(w)$ for the vocabulary set $\vocab=\{a,b,c\}$ (left) and the corresponding unigram table $T$ of size $10$ (right).}
\label{tab:unigram-table}
\end{center}
\end{table}

\begin{algorithm}[t]
\footnotesize
\caption{Adaptive unigram table.}
\begin{algorithmic}[1]
\STATE $f(w)\leftarrow 0$ for all $w\in\vocab$
\STATE $z\leftarrow 0$
\FOR{$i=1,\dots,n$}
\STATE $f(w_i)\leftarrow f(w_i)+1$
\STATE $F\leftarrow f(w_i)^{\alpha}-(f(w_i)-1)^{\alpha}$
\STATE $z\leftarrow z+F$
\IF{$|T|<\tau$}
\STATE add $F$ copies of $w_i$ to $T$
\ELSE
\FOR{$j=1,\dots,\tau$}
\STATE $T[j]\leftarrow w_i$ with probability $\frac{F}{z}$
\ENDFOR
\ENDIF
\ENDFOR
\end{algorithmic}
\label{alg:reservoir}
\end{algorithm}

\subsubsection{Adaptive unigram table} \label{subsec:reservoir}

Another problem is how to generate negative samples efficiently. Since $k$ negative samples per target-context pair have to be generated by the noise distribution, the sampling speed has a significant effect on the overall training efficiency. 

Let us first examine how negative samples are generated in batch SGNS\@. In a popular implementation \cite{Mikolov}, a word array (referred to as a \textit{unigram table}) is constructed such that the number of a word $w$ in it is proportional to $q(w)$. See Table~\ref{tab:unigram-table} for an example. Using the unigram table, negative samples can be efficiently generated by sampling the table elements uniformly at random. It takes only $O(1)$ time to generate one negative sample. 

The above method assumes that the noise distribution is fixed and thus cannot be used directly for the incremental training. One simple solution is to reconstruct the unigram table whenever new training data is provided. However, such a method is not effective for the incremental SGNS, because the unigram table reconstruction requires $\mathcal{O}(\vocabsize)$ time.%
\footnote{This overhead is amortized in mini-batch SGNS if the mini-batch size is sufficiently large. Our discussion here is dedicated to efficiently perform the incremental training irrespective of the mini-batch size.}

We propose a reservoir-based algorithm for efficiently updating the unigram table \cite{Vitter85,Efraimidis15} (Algorithm~\ref{alg:reservoir}). The algorithm incrementally update the unigram table $T$ while limiting its maximum size to $\tau$. In case $|T|<\tau$, it can be easily confirmed that the number of a word $w$ in $T$ is  $f(w)^{\alpha} (\propto q(w))$. In case $|T|=\tau$, since $z=\sum_{w\in\vocab}f(w)^{\alpha}$ is equal to the normalization factor of the noise distribution, it can be proven by induction that, for all $j$, $T[j]$ is a word $w$ with probability $q(w)$.
See \cite{Vitter85,Efraimidis15} for reference.

\paragraph{Note on implementation}

In line 8, $F$ copies of $w_i$ are added to $T$. When $F$ is not an integer, the copies are generated so that their expected number becomes $F$. Specifically, $\lceil F\rceil$ copies are added to $T$ with probability $F-\lfloor F\rfloor$, and $\lfloor F\rfloor$ copies are added otherwise.

The loop from line 10 to 12 becomes expensive if implemented straightforwardly because the maximum table size $\tau$ is typically set large (\textit{e.g.}, $\tau=10^8$ in \texttt{word2vec} \cite{Mikolov}). For acceleration, instead of checking all elements in the unigram table, randomly chosen $\frac{\tau F}{z}$ elements are substituted with $w_i$. Note that $\frac{\tau F}{z}$ is the expected number of table elements to be substituted in the original algorithm. This approximation achieves great speed-up because we usually have $F\ll z$.  In fact, it can be proven that it takes $O(1)$ time when $\alpha=1.0$. See Appendix%
\footnote{The appendices are in the supplementary material.} A for more discussions.

\subsection{Computational complexity} \label{subsec:isgns_complexity}

Both incremental and batch SGNS have the same space complexity, which is independent of the training data size~$n$. Both require $\mathcal{O}(\vocabsize)$ space to store the word embeddings and the word frequency counts, and $\mathcal{O}(|T|)$ space to store the unigram table.

The two algorithms also have the same time complexity. Both require $\mathcal{O}(n)$ training time when the training data size is $n$. Although incremental SGNS requires extra time for updating the dynamic vocabulary and adaptive unigram table, these costs are practically negligible, as will be demonstrated in Section~\ref{subsec:exp3}.

\section{Theoretical Analysis} \label{sec:analysis} 

Although the extension from batch to incremental SGNS is simple and intuitive, it is not readily clear whether incremental SGNS can learn word embeddings as well as the batch counterpart. To answer this question, in this section we examine incremental SGNS from a theoretical point of view.

The analysis begins by examining the difference between the objectives optimized by batch and incremental SGNS (Section \ref{subsec:obj_diff}). Then, probabilistic properties of their difference are investigated to demonstrate the relationship between batch and incremental SGNS (Sections~\ref{subsec:unsmoothed_case} and \ref{subsec:smoothed_case}). We shortly touch the mini-batch SGNS at the end of this section (Section~\ref{subsec:minibatch}).

\subsection{Objective difference} \label{subsec:obj_diff}

As discussed in Section \ref{sec:sgns}, batch SGNS optimizes the following objective:
\begin{align*}
\mathcal{L}_\text{B}(\theta)
\!=\!-\frac{1}{n}\sum_{i=1}^{n}\sum_{\substack{\mid j\mid\leq c\\j\ne 0}}\psi_{w_i,w_{i+j}}^{+}\!+\!k\mathbb{E}_{v\sim q_n(v)}[\psi_{w_i,v}^{-}],
\end{align*}
where $\theta=(\tvec{1},\tvec{2},\dots,\tvec{\vocabsize},\cvec{1},\cvec{2},\dots,\cvec{\vocabsize})$ collectively represents the model parameter%
\footnote{We treat words as integers and thus $\vocab\!=\!\{1,2,\dots |\vocab|\}$.}  (\textit{i.e.}, word embeddings) and $q_n(v)$ represents the noise distribution. Note that the noise distribution is represented in a different notation than Section~\ref{sec:sgns} to make its dependence on the whole training data explicit. The function $q_i(v)$ is defined as $q_i(v)=\frac{f_i(v)^{\alpha}}{\sum_{v'\in\vocab}f_i(v')^{\alpha}}$, where $f_i(v)$ represents the word frequency in the first $i$ words of the training data.

In contrast, incremental SGNS computes the gradient of $-\psi_{w_i,w_{i+j}}^{+}-k\mathbb{E}_{v\sim q_i(v)}[\psi_{w_i,v}^{-}]$ at each step to perform gradient descent. Note that the noise distribution does not depend on $n$ but rather on $i$. Because it can be seen as a sample approximation of the gradient of
\begin{align*}
\mathcal{L}_\text{I}(\theta)
=-\frac{1}{n}\sum_{i=1}^{n}\sum_{\substack{\mid j\mid\leq c\\j\ne 0}}\psi_{w_i,w_{i+j}}^{+}\!+\!k\mathbb{E}_{v\sim q_i(v)}[\psi_{w_i,v}^{-}],
\end{align*}
incremental SGNS can be interpreted as optimizing $\mathcal{L}_\text{I}(\theta)$ with SGD.

Since the expectation terms in the objectives can be rewritten as $\mathbb{E}_{v\sim q_i(v)}[\psi_{w_i,v}^{-}]=\sum_{v\in\vocab}q_i(v)\psi_{w_i,v}^{-}$, the difference between the two objectives can be given as
\begin{align*}
\Delta\mathcal{L}(\theta)
&=\mathcal{L}_\text{B}(\theta)-\mathcal{L}_\text{I}(\theta)\\
&=\frac{1}{n}\sum_{i=1}^{n}\sum_{\substack{|j|\leq c\\j\ne 0}}\!k\!\sum_{v\in\vocab}(q_i(v)\!-\!q_n(v))\psi_{w_i,v}^{-}\\
&=\frac{2ck}{n}\sum_{i=1}^{n}\sum_{v\in\vocab}(q_i(v)-q_n(v))\psi_{w_i,v}^{-}\\
&=\frac{2ck}{n}\!\!\sum_{w,v\in\vocab}\sum_{i=1}^n\delta_{w_i,w}(q_i(v)-q_n(v))\psi_{w,v}^{-}
\end{align*}
where $\delta_{w,v}=\delta(w=v)$ is the delta function.

\subsection{Unsmoothed case} \label{subsec:unsmoothed_case}

Let us begin by examining the objective difference $\Delta\mathcal{L}(\theta)$ in the unsmoothed case, $\alpha=1.0$.

The technical difficulty in analyzing $\Delta\mathcal{L}(\theta)$ is that it is dependent on the word order in the training data. To address this difficulty, we assume that the words in the training data are generated from some stationary distribution. This assumption allows us to investigate the property of $\Delta\mathcal{L}(\theta)$ from a probabilistic perspective. Regarding the validity of this assumption, we want to note that this assumption is already taken by the original SGNS: the probability that the target and context words co-occur is assumed to be independent of their position in the training data.

We below introduce some definitions and notations as the preparation of the analysis.

\begin{definition}
Let $\rv{X}{i,w}$ be a random variable that represents $\delta_{w_i,w}$. It takes $1$ when the $i$-th word in the training data is $w\in\vocab$ and $0$ otherwise.
\end{definition}

Remind that we assume that the words in the training data are generated from a stationary distribution. This assumption means that the expectation and (co)variance of $\rv{X}{i,w}$ do not depend on the index $i$. Hereafter, they are respectively denoted as $\mathbb{E}[\rv{X}{i,w}]=\mu_w$ and $\mathbb{V}[\rv{X}{i,w},\rv{X}{j,v}]=\var{w,v}$.

\begin{definition}
Let $\rv{Y}{i,w}$ be a random variable that represents $q_i(w)$ when $\alpha=1.0$. It is given as $\rv{Y}{i,w}=\frac{1}{i}\sum_{i'=1}^{i}\rv{X}{i',w}$.
\end{definition}

\subsubsection{Convergence of the first and second order moments of $\Delta\mathcal{L}(\theta)$} 
\label{subsec:moment_convergence}

It can be shown that the first order moment of $\Delta\mathcal{L}(\theta)$ has an analytical form.
\begin{theorem}
\label{theorem:first_moment}
The first order moment of $\Delta\mathcal{L}(\theta)$ is given as
\begin{align*}
	\mathbb{E}[\Delta\mathcal{L}(\theta)]=\frac{2ck(H_{n}-1)}{n}\sum_{w,v\in\vocab}\var{w,v}\psi_{w,v}^{-},
\end{align*}
where $H_{n}$ is the $n$-th harmonic number.
\begin{proof}[Sketch of proof]
Notice that $\mathbb{E}[\Delta\mathcal{L}(\theta)]$ can be written as
\begin{align*}
\frac{2ck}{n}\sum_{w,v\in\vocab}\sum_{i=1}^{n}\bigl(\mathbb{E}[\rv{X}{i,w}\rv{Y}{i,v}]-\mathbb{E}[\rv{X}{i,w}\rv{Y}{n,v}]\bigr)\psi_{w,v}^{-}.
\end{align*}
Because we have, for any $i$ and $j$ such that $i\leq j$,
\begin{align*}
\mathbb{E}[\rv{X}{i,w}\rv{Y}{j,v}]
=\!\sum_{j'=1}^{j}\mathbb{E}[\rv{X}{i,w}\frac{\rv{X}{j',v}}{j}]
=\mu_w\mu_v +\frac{\var{w,v}}{j},
\end{align*}
plugging this into $\mathbb{E}[\Delta\mathcal{L}(\theta)]$ proves the theorem. See Appendix B.1 for the complete proof.
\end{proof}
\end{theorem}

Theorem \ref{theorem:first_moment} readily gives the convergence property of the first order moment of $\Delta\mathcal{L}(\theta)$:
\begin{theorem}
\label{theorem:first_moment_convergence}
The first-order moment of $\Delta\mathcal{L}(\theta)$ decreases in the order of $\mathcal{O}(\frac{\log(n)}{n})$:
\begin{align*}
	\mathbb{E}[\Delta\mathcal{L}(\theta)]=\mathcal{O}\biggl(\frac{\log(n)}{n}\biggr),
\end{align*}
and thus converges to zero in the limit of infinity:
\begin{align*}
\lim_{n\rightarrow\infty}\mathbb{E}[\Delta\mathcal{L}(\theta)]=0.
\end{align*}
\begin{proof}
We have $H_{n}\!\!=\!\mathcal{O}(\log(n))$ from the upper integral bound, and thus Theorem 1 gives the proof.
\end{proof}
\end{theorem}

A similar result to Theorem \ref{theorem:first_moment_convergence} can be obtained for the second order moment of $\Delta\mathcal{L}(\theta)$ as well.
\begin{theorem}
\label{theorem:second_moment_convergence}
The second-order moment of $\Delta\mathcal{L}(\theta)$ decreases in the order of $\mathcal{O}(\frac{\log(n)}{n})$:
\begin{align*}
\mathbb{E}[\Delta\mathcal{L}(\theta)^2]=\mathcal{O}\biggl(\frac{\log(n)}{n}\biggr),
\end{align*}
and thus converges to zero in the limit of infinity:
\begin{align*}
\lim_{n\rightarrow\infty}\mathbb{E}[\Delta\mathcal{L}(\theta)^2]=0.
\end{align*}
\begin{proof}
Omitted. See Appendix B.2.
\end{proof}
\end{theorem}

\subsubsection{Main result} \label{subsec:main_result}

The above theorems reveal the relationship between the optimal solutions of the two objectives, as stated in the next lemma.
\begin{lemma}
Let $\theta^*$ and $\hat{\theta}$ be the optimal solutions of $\mathcal{L}_\text{B}(\theta)$ and $\mathcal{L}_\text{I}(\theta)$, respectively: $\theta^{*}=\arg\min_{\theta}\mathcal{L}_\text{B}(\theta)$ and $\hat{\theta}=\arg\min_{\theta}\mathcal{L}_\text{I}(\theta)$. Then,
\begin{align}
\lim_{n\rightarrow\infty}\mathbb{E}[\mathcal{L}_\text{B}(\hat{\theta})-\mathcal{L}_\text{B}(\theta^*)]&=0,\label{eq:expect-diff}\\
\lim_{n\rightarrow\infty}\mathbb{V}[\mathcal{L}_\text{B}(\hat{\theta})-\mathcal{L}_\text{B}(\theta^*)]&=0.\label{eq:var-diff}
\end{align}
\begin{proof}
The proof is made by the squeeze theorem. Let $\l=\mathcal{L}_\text{B}(\hat{\theta})-\mathcal{L}_\text{B}(\theta^*)$. The optimality of $\theta^*$ gives $0\leq\l$. Also, the optimality of $\hat{\theta}$ gives
\begin{align}
\l&=\mathcal{L}_\text{B}(\hat{\theta})-\mathcal{L}_\text{I}(\theta^{*})+\mathcal{L}_\text{I}(\theta^{*})-\mathcal{L}_\text{B}(\theta^*)\nonumber\\
&\leq\mathcal{L}_\text{B}(\hat{\theta})-\mathcal{L}_\text{I}(\hat{\theta})+\mathcal{L}_\text{I}(\theta^*)-\mathcal{L}_\text{B}(\theta^*)\nonumber\\
    &=\Delta\mathcal{L}(\hat{\theta})-\Delta\mathcal{L}(\theta^*).\nonumber
\end{align}
We thus have $0\leq\mathbb{E}[\l]\leq\mathbb{E}[\Delta\mathcal{L}(\hat{\theta})-\Delta\mathcal{L}(\theta^*)]$. Since Theorem \ref{theorem:first_moment_convergence} implies that the right hand side converges to zero when $n\rightarrow\infty$, the squeeze theorem gives Eq.~(\ref{eq:expect-diff}). Next, we have
\begin{align}
\mathbb{V}[\l]
&\!=\!\mathbb{E}[\l^2]\!-\!\mathbb{E}[\l]^2\!\leq\!\mathbb{E}[\l^2]\!\leq\!\mathbb{E}[(\Delta\mathcal{L}(\hat{\theta})\!-\!\Delta\mathcal{L}(\theta^*))^2]\nonumber\\
&\!\leq\!\mathbb{E}[(\Delta\mathcal{L}(\hat{\theta})\!-\!\Delta\mathcal{L}(\theta^*))^2
\!\!+\!\!(\Delta\mathcal{L}(\hat{\theta})\!+\!\Delta\mathcal{L}(\theta^{*}))^2]\nonumber\\
&\!=\!2\mathbb{E}[\Delta\mathcal{L}(\hat{\theta})^2]+2\mathbb{E}[\Delta\mathcal{L}(\theta^*)^2].\label{eq:upper}
\end{align}
Theorem \ref{theorem:second_moment_convergence} suggests that Eq.~(\ref{eq:upper}) converges to zero when $n\rightarrow\infty$. Also, the non-negativity of the variance gives $0\leq\mathbb{V}[\l]$. Therefore, the squeeze theorem gives Eq.~(\ref{eq:var-diff}).
\end{proof}
\end{lemma}

We are now ready to provide the main result of the analysis. The next theorem shows the convergence of $\mathcal{L}_\text{B}(\hat{\theta})$.
\begin{theorem}
$\mathcal{L}_\text{B}(\hat{\theta})$ converges in probability to $\mathcal{L}_{\text{B}}(\theta^*)$:
\begin{align*}
\forall\epsilon>0, \lim_{n\rightarrow\infty}\Pr\biggl[|\mathcal{L}_\text{B}(\hat{\theta})-\mathcal{L}_\text{B}(\theta^*)|\geq\epsilon\biggr]=0.
\end{align*}
\end{theorem}
\begin{proof}
Let again $\l=\mathcal{L}_\text{B}(\hat{\theta})-\mathcal{L}_\text{B}(\theta^*)$. Then, Chebyshev's inequality gives, for any $\epsilon_1>0$, 
\begin{align*}
\lim_{n\rightarrow\infty}\frac{\mathbb{V}[\l]}{\epsilon_1^2}
\geq\lim_{n\rightarrow\infty}\Pr\biggl[|\l-\mathbb{E}[\l]|\geq\epsilon_1\biggr].
\end{align*}
Remember that Eq.~(\ref{eq:expect-diff}) means that for any $\epsilon_2>0$, there exists $n'$ such that if $n'\leq n$ then $|\mathbb{E}[\l]|<\epsilon_2$. Therefore, we have
\begin{align*}
\lim_{n\rightarrow\infty}\frac{\mathbb{V}[\l]}{\epsilon_1^2}\geq\lim_{n\rightarrow\infty}\Pr\biggl[|\l|\geq\epsilon_1+\epsilon_2\biggr]\geq 0.
\end{align*}
The arbitrary property of $\epsilon_1$ and $\epsilon_2$ allows $\epsilon_1+\epsilon_2$ to be rewritten as $\epsilon$. Also, Eq.~(\ref{eq:var-diff}) implies that $\lim_{n\rightarrow\infty}\frac{\mathbb{V}[\l]}{\epsilon_1^2}=0$. This completes the proof.
\end{proof}
\noindent
Informally, this theorem can be interpreted as suggesting that the optimal solutions of batch and incremental SGNS agree when $n$ is infinitely large.

\subsection{Smoothed case} \label{subsec:smoothed_case}

We next examine the smoothed case ($0<\alpha< 1$). In this case, the noise distribution can be represented by using the ones in the unsmoothed case:
\begin{align*}
q_i(w)=\frac{f_i(w)^{\alpha}}{\sum_{w'\in\vocab}f_i(w')^{\alpha}}
=\frac{\bigl(\frac{f_i(w)}{F_i}\bigr)^{\alpha}}{\sum_{w'\in\vocab}\bigl(\frac{f_i(w')}{F_i}\bigr)^{\alpha}}
\end{align*}
where $F_i=\sum_{w'\in\vocab}f_i(w')$ and $\frac{f_i(w)}{F_i}$ corresponds to the noise distribution in the unsmoothed case.
\begin{definition}
Let $\rv{Z}{i,w}$ be a random variable that represents $q_i(w)$ in the smoothed case. Then, it can be written by using $\rv{Y}{i,w}$:
\begin{align*}
\rv{Z}{i,w}=g_w(\rv{Y}{i,1},\rv{Y}{i,2},\dots,\rv{Y}{i,\vocabsize})
\end{align*}
where $g_w(x_1,x_2,\dots,x_{\vocabsize})=\frac{x_w^{\alpha}}{\sum_{w'\in\vocab}x_{w'}^{\alpha}}$.
\end{definition}

Because $\rv{Z}{i,w}$ is no longer a linear combination of $\rv{X}{i,w}$, it becomes difficult to derive similar proofs to the unsmoothed case. To address this difficulty, $\rv{Z}{i,w}$ is approximated by the first-order Taylor expansion around
\[
\mathbb{E}[(\rv{Y}{i,1},\rv{Y}{i,2},\dots,\rv{Y}{i,\vocabsize})]=(\mu_1,\mu_2,\dots,\mu_{\vocabsize}).
\]
The first-order Taylor approximation gives
\begin{align*}
\rv{Z}{i,w}\approx g_w(\mu)+\sum_{v\in\vocab}M_{w,v}(\rv{Y}{i,v}-g_{v}(\mu))
\end{align*}
where $\mu=(\mu_1,\mu_2,\dots,\mu_{\vocabsize})$ and $M_{w,v}=\frac{\partial g_w(x)}{\partial x_v}|_{x=\mu}$. Consequently, it can be shown that the first and second order moments of $\Delta\mathcal{L}(\theta)$ have the order of $\mathcal{O}(\frac{\log(n)}{n})$ in the smoothed case as well. See Appendix~C for the details.

\subsection{Mini-batch SGNS} \label{subsec:minibatch}

The same analysis result can also be obtained for the mini-batch SGNS\@. We can prove Theorems 2 and 3 in the mini-batch case as well (see Appendix~D for the proof). The other part of the analysis remains the same.

\section{Experiments}

Three experiments were conducted to investigate the correctness of the theoretical analysis (Section~\ref{subsec:exp1}) and the practical usefulness of incremental SGNS (Sections~\ref{subsec:exp2} and \ref{subsec:exp3}). Details of the experimental settings that do not fit into the paper are presented in Appendix E.

\subsection{Validation of theorems} \label{subsec:exp1}

An empirical experiment was conducted to validate the result of the theoretical analysis. Since it is difficult to assess the main result in Section \ref{subsec:main_result} directly, the theorems in Sections \ref{subsec:moment_convergence}, from which the main result is readily derived, were investigated. Specifically, the first and second order moments of $\Delta\mathcal{L}(\theta)$ were computed on datasets of increasing sizes to empirically investigate the convergence property.

Datasets of various sizes were constructed from the English Gigaword corpus \cite{Napoles12}. The datasets made up of $n$ words were constructed by randomly sampling sentences from the Gigaword corpus. The value of $n$ was varied over $\{10^3,10^4,10^5,10^6,10^7\}$. $10,000$ different datasets were created for each size $n$ to compute the first and second order moments.

Figure \ref{fig:moment} (top left) shows log-log plots of the first order moments of $\Delta\mathcal{L}(\theta)$ computed on the different sized datasets when $\alpha=1.0$. The crosses and circles represent the empirical values and theoretical values obtained by Theorem 1, respectively. Figure \ref{fig:moment} (top right) similarly illustrates the second order moments of $\Delta\mathcal{L}(\theta)$. Since Theorem~\ref{theorem:second_moment_convergence} suggests that the second order moment decreases in the order of $\mathcal{O}(\frac{\log(n)}{n})$, the graph $y\propto\frac{\log(x)}{x}$ is also shown. The graph was fitted to the empirical data by minimizing the squared error.

The top left figure demonstrates that the empirical values of the first order moments fit the theoretical result very well, providing a strong empirical evidence for the correctness of Theorem~\ref{theorem:first_moment}. In addition, the two figures show that the first and second order moments decrease almost in the order of $\mathcal{O}(\frac{\log(n)}{n})$, converging to zero as the data size increases. This result validates Theorems~\ref{theorem:first_moment_convergence} and \ref{theorem:second_moment_convergence}.

Figures~\ref{fig:moment}  (bottom left) and (bottom right) show similar results when $\alpha=0.75$. Since we do not have theoretical estimates of the first order moment when $\alpha\not=1.0$, the graphs $y\propto\frac{\log(n)}{n}$ are shown in both figures. From these, we can again observe that the first and second order moments decrease almost in the order of $\mathcal{O}(\frac{\log(n)}{n})$. This indicates the validity of the investigation in Section~\ref{subsec:smoothed_case}. The relatively larger deviations from the graphs $y\propto\frac{\log(n)}{n}$, compared with the top right figure, are considered to be attributed to the first-order Taylor approximation.

\begin{figure}[t]
\centering
\begin{tabular}[2]{@{}c@{}c@{}}
\includegraphics[width=0.25\textwidth]{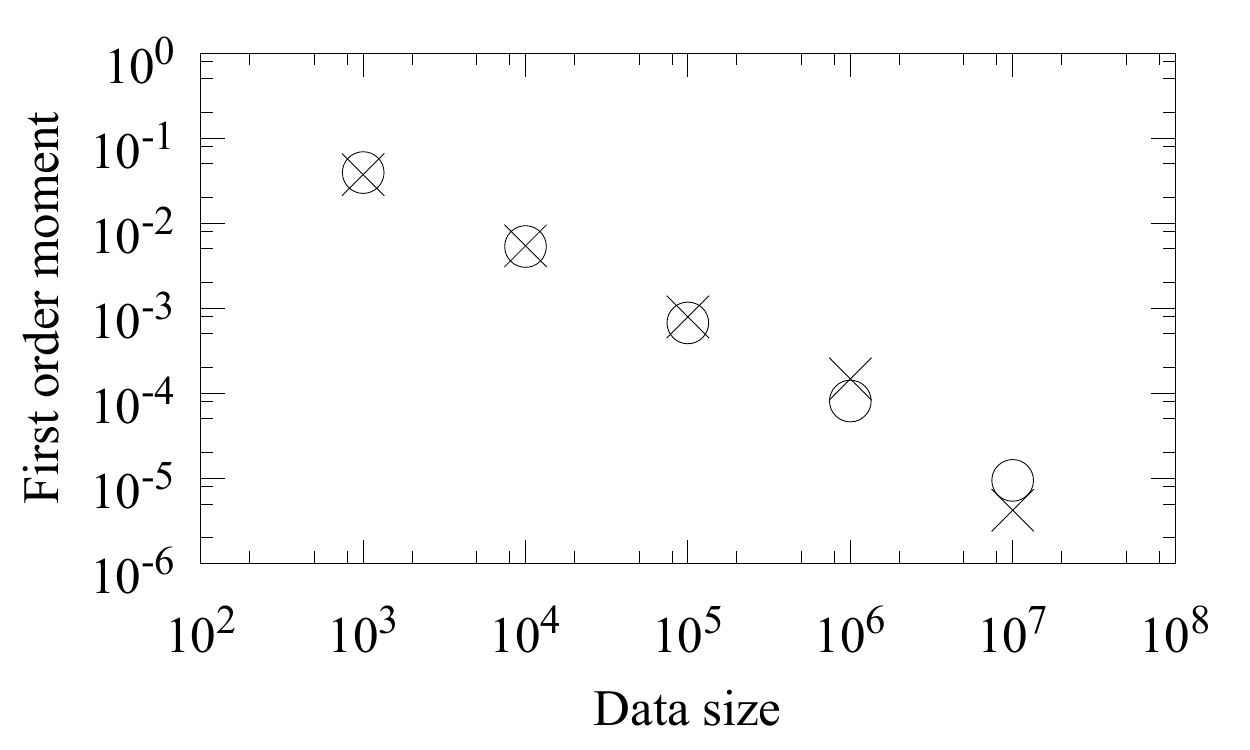} &
\includegraphics[width=0.25\textwidth]{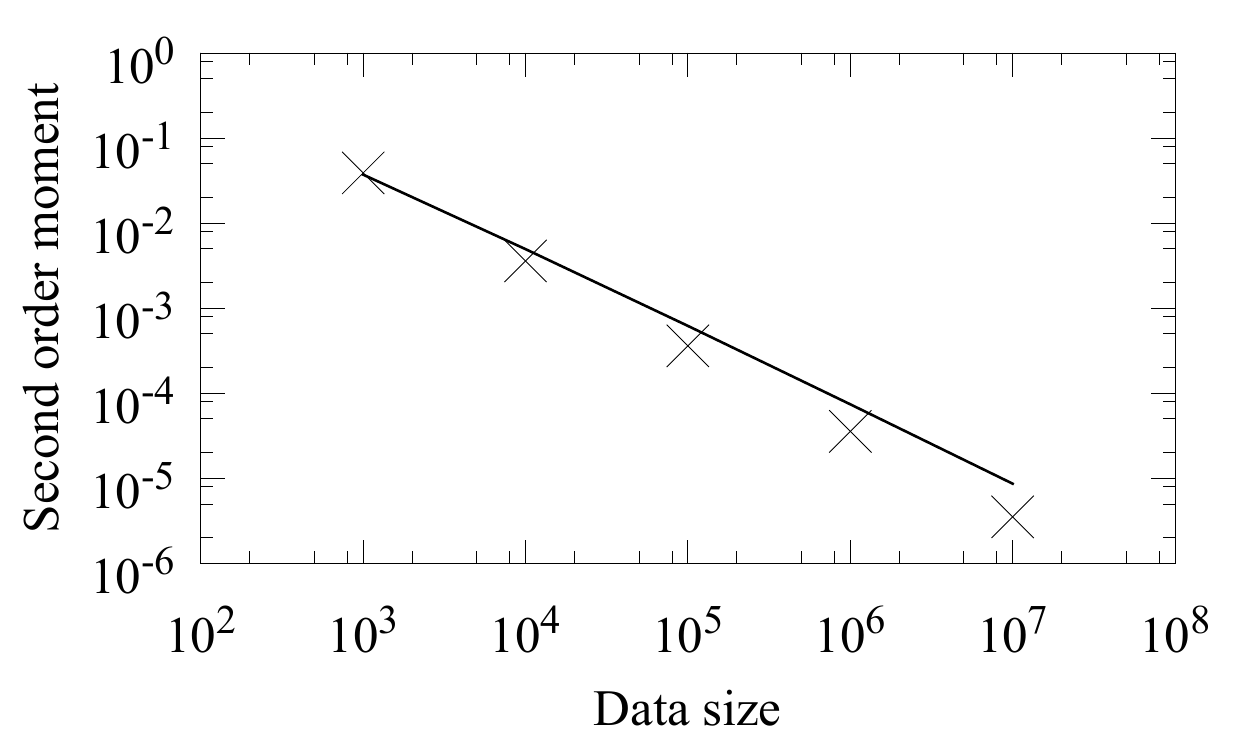} \\
\includegraphics[width=0.25\textwidth]{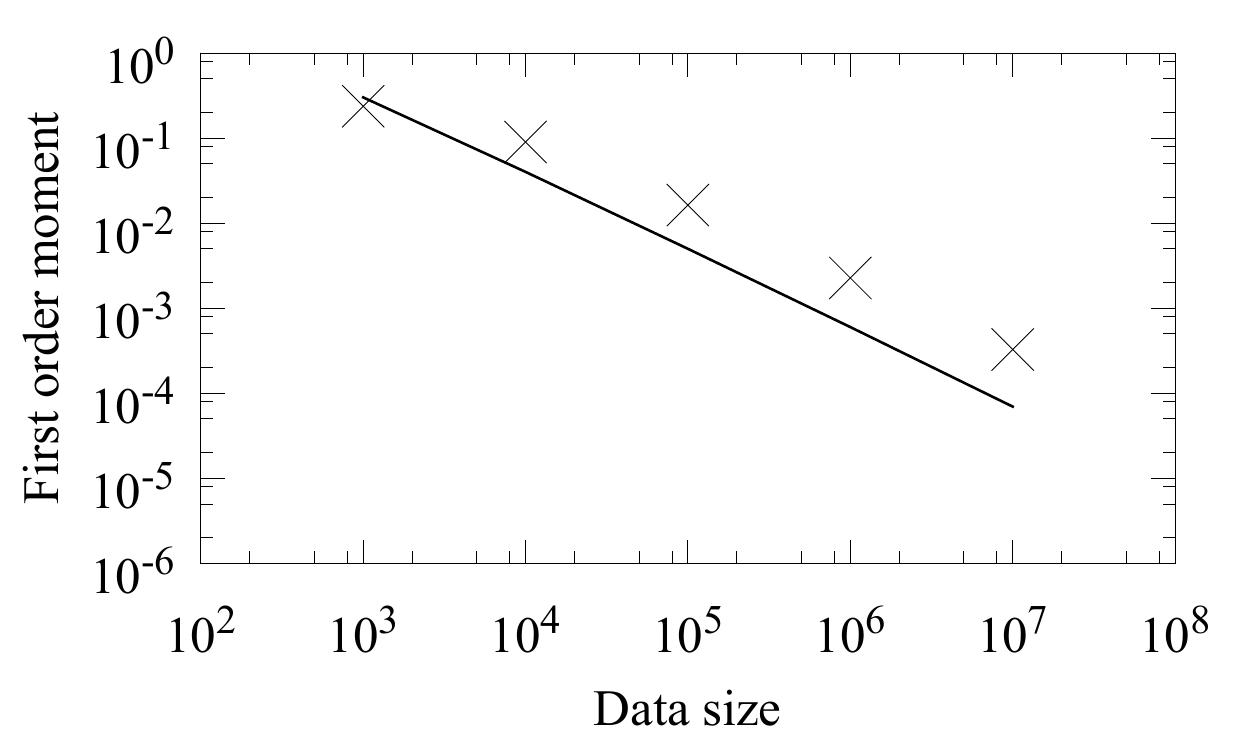} &
\includegraphics[width=0.25\textwidth]{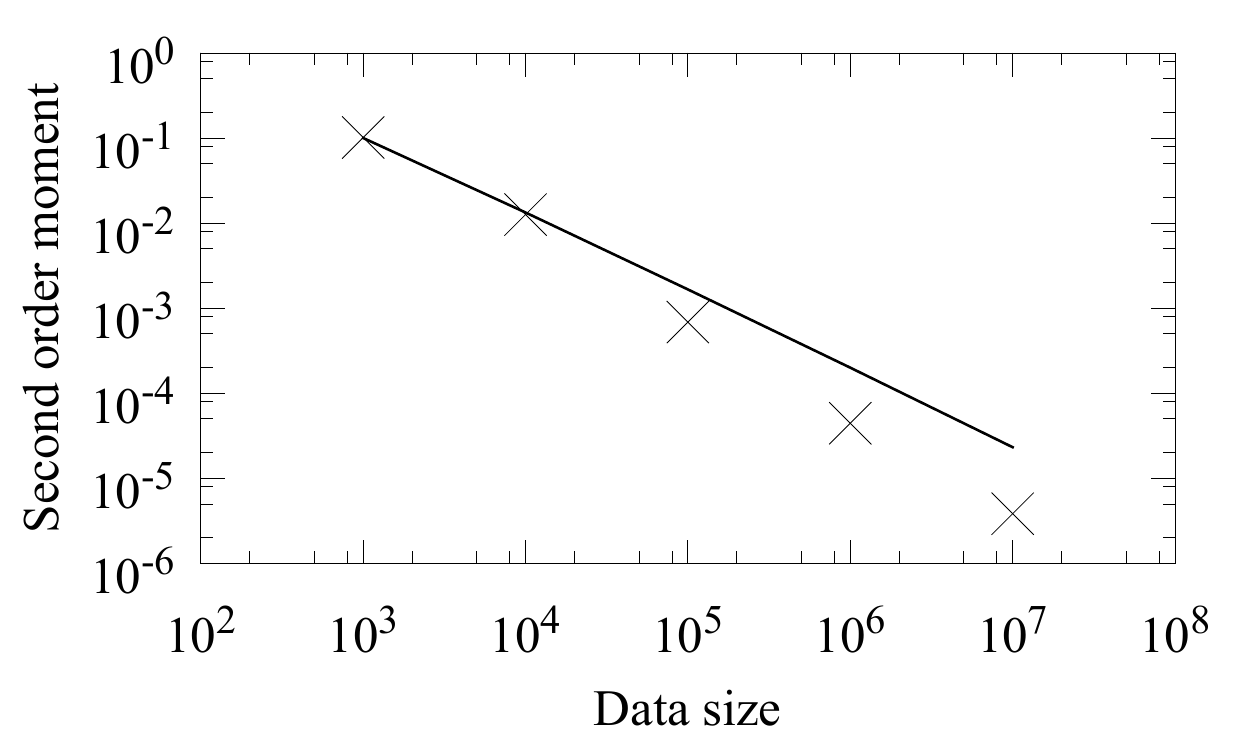} 
\end{tabular}  
\caption{Log-log plots of the first and second order moments of $\Delta\mathcal{L}(\theta)$ on the different sized datasets when $\alpha=1.0$ (top left and top right) and $\alpha=0.75$ (bottom left and bottom right).}
\label{fig:moment}
\end{figure}

\subsection{Quality of word embeddings} \label{subsec:exp2}

The next experiment investigates the quality of the word embeddings learned by incremental SGNS through comparison with the batch counterparts.

The Gigaword corpus was used for the training. For the comparison, both our own implementation of batch SGNS as well as \textsc{word2vec} \cite{Mikolov13a} were used (denoted as \textbf{batch} and \textbf{w2v}). The training configurations of the three methods were set the same as much as possible, although it is impossible to do so perfectly. For example, incremental SGNS (denoted as \textbf{incremental}) utilized the dynamic vocabulary (\textit{c.f.}, Section \ref{subsec:dynamic_vocab}) and thus we set the maximum vocabulary size $m$ to control the vocabulary size. On the other hand, we set a frequency threshold to determine the vocabulary size of \textbf{w2v}. We set $m=240$k for \textbf{incremental}, while setting the frequency threshold to $100$ for \textbf{w2v}. This yields vocabulary sets of comparable sizes: $220,389$ and $246,134$. 

The learned word embeddings were assessed on five benchmark datasets commonly used in the literature \cite{Levy15}: WordSim353 \cite{Agirre09}, MEN \cite{Bruni13}, SimLex-999 \cite{Hill15}, the MSR analogy dataset \cite{Mikolov13a}, the Google analogy dataset \cite{Mikolov13b}. The former three are for a semantic similarity task, and the remaining two are for a word analogy task. As evaluation measures, Spearman's $\rho$ and prediction accuracy were used in the two tasks, respectively.

Figures~\ref{fig:result} (a) and (b) represent the results on the similarity datasets and the analogy datasets. We see that the three methods (\textbf{incremental}, \textbf{batch}, and \textbf{w2v}) perform equally well on all of the datasets. This indicates that incremental SGNS can learn as good word embeddings as the batch counterparts, while being able to perform incremental model update. Although \textbf{incremental} performs slightly better than the batch methods in some datasets, the difference seems to be a product of chance.

The figures also show the results of incremental SGNS when the maximum vocabulary size $m$ was reduced to $150$k and $100$k (\textbf{incremental-150k} and \textbf{incremental-100k}). The resulting vocabulary sizes were $135,447$ and $86,993$, respectively. We see that \textbf{incremental-150k} and \textbf{incremental-100k} perform comparatively well with \textbf{incremental}, although relatively large performance drops are observed in some datasets (MEN and MSR). This demonstrates that the Misra-Gries algorithm can effectively control the vocabulary size.

\begin{figure*}[t]
\centering
\begin{tabular}[3]{@{}c@{}c@{}c@{}}
 \includegraphics[width=0.33\textwidth]{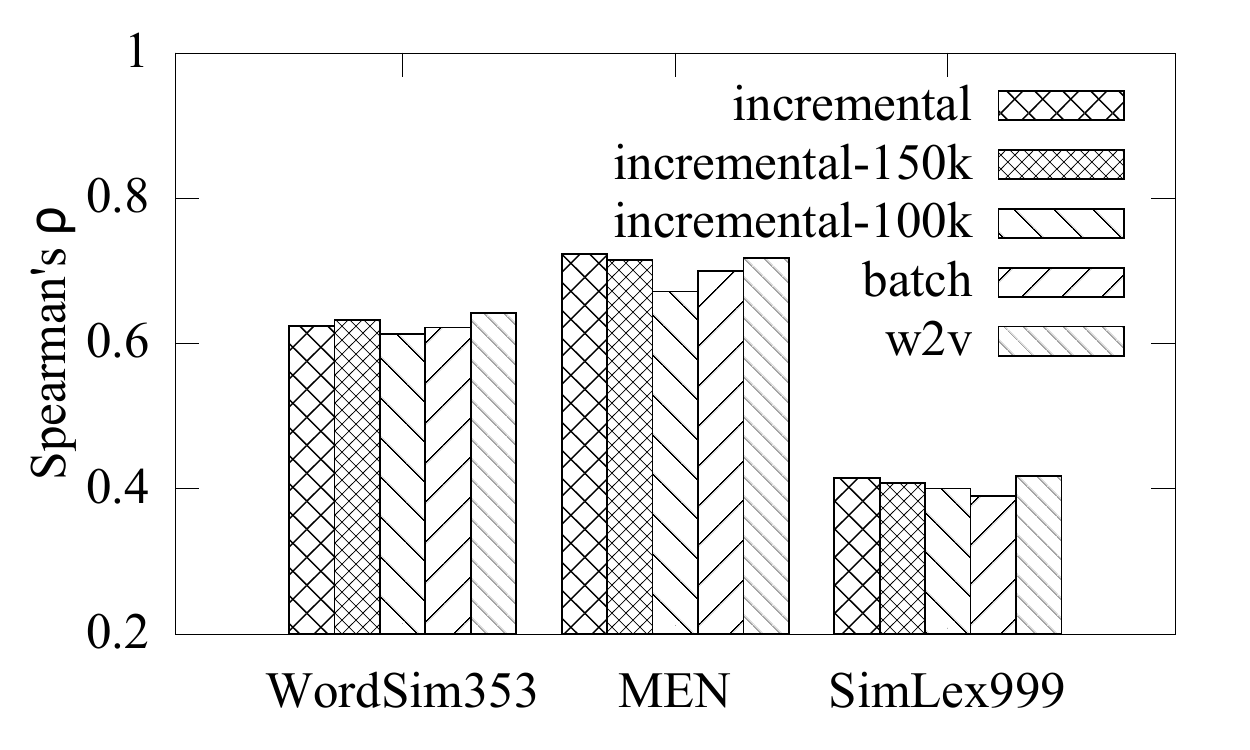} &
 \includegraphics[width=0.33\textwidth]{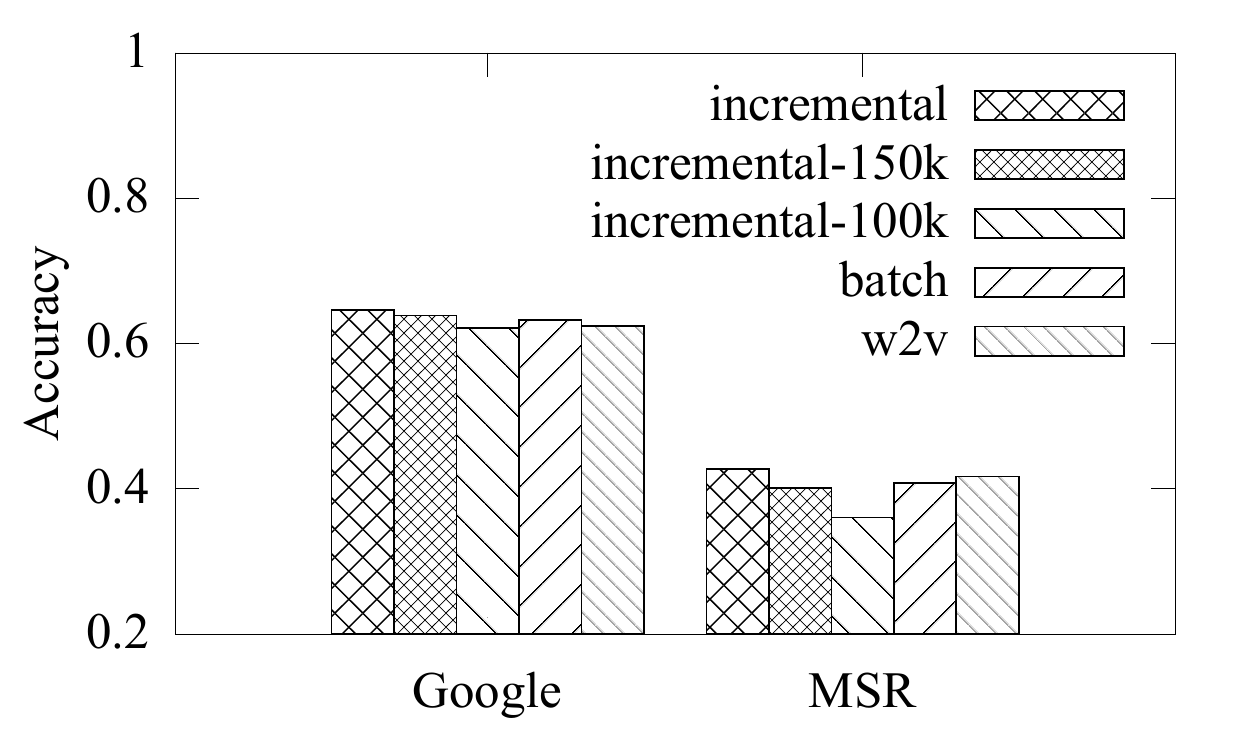} &
\includegraphics[width=0.33\textwidth]{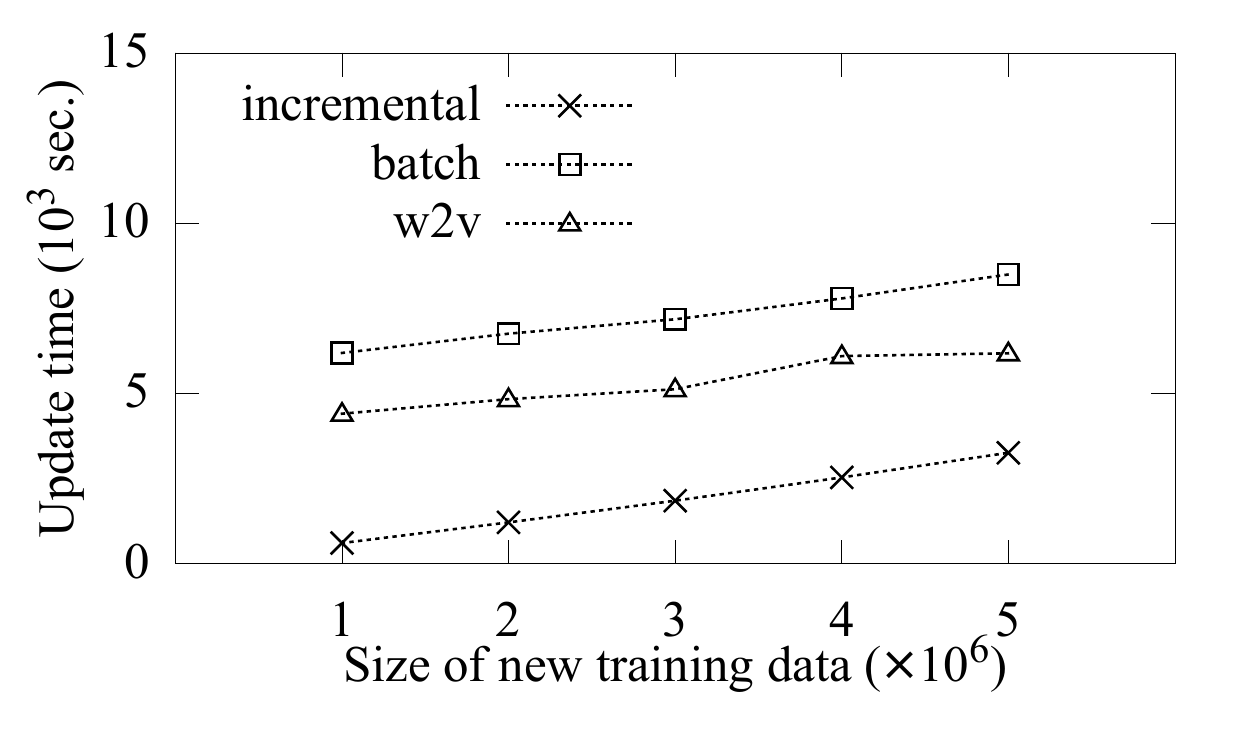} \\
(a) & (b) & (c) \\
\end{tabular}  
\caption{(a): Spearman's $\rho$ on the word similarity datasets. (b): Accuracy on the analogy datasets. (c): Update time when new training data is provided.}
\label{fig:result}
\end{figure*}

\subsection{Update time} \label{subsec:exp3}

The last experiment investigates how much time incremental SGNS can save by avoiding re-training when updating the word embeddings.

In this experiment, \textbf{incremental} was first trained on the initial training data of size%
\footnote{The number of sentences here.} $n_1$ and then updated on the new training data of size $n_2$ to measure the update time. For comparison, \textbf{batch} and \textbf{w2v} were re-trained on the combination of the initial and new training data. We fixed $n_1=10^7$ and varied $n_2$ over $\{1\times10^6,  2\times10^6,\dots, 5\times10^6\}$.

Figure~\ref{fig:result} (c) compares the update time of the three methods across various values of $n_2$. We see that \textbf{incremental} significantly reduces the update time. It achieves $10$ and $7.3$ times speed-up compared with \textbf{batch} and \textbf{w2v} (when $n_2=10^6$). This represents the advantage of the incremental algorithm, as well as the time efficiency of the dynamic vocabulary and adaptive unigram table. We note that \textbf{batch} is slower than \textbf{w2v} because it uses AdaGrad, which maintains different learning rates for different dimensions of the parameter, while \textbf{w2v} uses the same learning rate for all dimensions.

\section{Related Work}

Word representations based on distributional semantics have been common \cite{Turney10,Baroni10}. The distributional methods typically begin by constructing a word-context matrix and then applying dimension reduction techniques such as SVD to obtain high-quality word meaning representations. Although some investigated incremental updating of the word-context matrix \cite{Yin15,Goyal11}, they did not explore the reduced representations.
On the other hand, neural word embeddings have recently gained much popularity as an alternative. However, most previous studies have not explored incremental strategies \cite{Mikolov13b,Mikolov13c,Pennington14}.

Very recently, \citet{Peng17} proposed an incremental learning method of hierarchical soft-max. Because hierarchical soft-max and negative sampling have different advantages \cite{Peng17}, the incremental SGNS and their method are complementary to each other. Also, their updating method needs to scan not only new but also old training data, and thus is not an incremental algorithm in a strict sense. As a consequence, it potentially incurs the same time complexity as the re-training. 
Another consequence is that their method has to retain the old training data and thus wastes space, while incremental SGNS can discard old training examples after processing them.

There are publicly available implementations for training SGNS, one of the most popular being \textsc{word2vec} \cite{Mikolov}. However, it does not support an incremental training method. \textsc{Gensim} \cite{Rehurek10} also offers SGNS training. Although \textsc{Gensim} allows the incremental updating of SGNS models, it is done in an ad-hoc manner. In \textsc{Gensim}, the vocabulary set as well as the unigram table are fixed once trained, meaning that new words cannot be added. Also, they do not provide any theoretical accounts for the validity of their training method.

\section{Conclusion and Future Work}

This paper proposed incremental SGNS and provided thorough theoretical analysis to demonstrate its validity. We also conducted experiments to empirically demonstrate its effectiveness. Although the incremental model update is often required in practical machine learning applications, only a little attention has been paid to learning word embeddings incrementally. We consider that incremental SGNS successfully addresses this situation and serves as an useful tool for practitioners.

The success of this work suggests several research directions to be explored in the future. One possibility is to explore extending other embedding methods such as GloVe  \cite{Pennington14} to incremental algorithms. Such studies would further extend the potential of word embedding methods.

\bibliographystyle{emnlp_natbib}
\bibliography{emnlp2017}  

\begin{thebibliography}{}
\expandafter\ifx\csname natexlab\endcsname\relax\def\natexlab#1{#1}\fi

\bibitem[{Agirre et~al.(2009)Agirre, Alfonseca, Hall, Kravalova, Pasca, and
  Soroa}]{Agirre09}
Eneko Agirre, Enrique Alfonseca, Keith Hall, Jana Kravalova, Marius Pasca, and
  Aitor Soroa. 2009.
\newblock \href{http://www.aclweb.org/anthology/N/N09/N09-1003}{A study on
  similarity and relatedness using distributional and wordnet-based
  approaches}.
\newblock In {\em Proceedings of NAACL\/}. pages 19--27.
\newblock
  \href{http://www.aclweb.org/anthology/N/N09/N09-1003}{http://www.aclweb.org/anthology/N/N09/N09-1003}.

\bibitem[{Baroni and Lenci(2010)}]{Baroni10}
Marco Baroni and Alessandro Lenci. 2010.
\newblock \href{http://aclweb.org/anthology/J/J10/J10-4006}{Distributional
  memory: A general framework for corpus-based semantics}.
\newblock {\em Computatoinal Linguistics\/} 36:673--721.
\newblock
  \href{http://aclweb.org/anthology/J/J10/J10-4006}{http://aclweb.org/anthology/J/J10/J10-4006}.

\bibitem[{Bruni et~al.(2013)Bruni, Tran, and Baroni}]{Bruni13}
E.~Bruni, N.~K. Tran, and M.~Baroni. 2013.
\newblock Multimodal distributional semantics.
\newblock {\em Journal of Artificial Intelligence Research\/} 49:1--49.

\bibitem[{Duchi et~al.(2011)Duchi, Hazan, and Singer}]{Duchi11}
John Duchi, Elad Hazan, and Yoram Singer. 2011.
\newblock Adaptive subgradient methods for online learning and stochastic
  optimization.
\newblock {\em Journal of Machine Learning Research\/} 12:2121--2159.

\bibitem[{Efraimidis(2015)}]{Efraimidis15}
Pavlos~S. Efraimidis. 2015.
\newblock Weighted random sampling over data streams.
\newblock ArXiv:1012.0256.

\bibitem[{Goyal and Daume~III(2011)}]{Goyal11}
Amit Goyal and Hal Daume~III. 2011.
\newblock \href{http://www.aclweb.org/anthology/D11-1023}{Approximate scalable
  bounded space sketch for large data nlp}.
\newblock In {\em Proceedings of EMNLP\/}. pages 250--261.
\newblock
  \href{http://www.aclweb.org/anthology/D11-1023}{http://www.aclweb.org/anthology/D11-1023}.

\bibitem[{Hill et~al.(2015)Hill, Reichart, and Korhonen}]{Hill15}
Felix Hill, Roi Reichart, and Anna Korhonen. 2015.
\newblock \href{http://aclweb.org/anthology/J/J15/J15-4004}{Simlex-999:
  Evaluating semantic models with (genuine) similarity estimation}.
\newblock {\em Computational Linguistics\/} 41:665--695.
\newblock
  \href{http://aclweb.org/anthology/J/J15/J15-4004}{http://aclweb.org/anthology/J/J15/J15-4004}.

\bibitem[{Levy et~al.(2015)Levy, Goldberg, and Dagan}]{Levy15}
Omer Levy, Yoav Goldberg, and Ido Dagan. 2015.
\newblock
  \href{https://tacl2013.cs.columbia.edu/ojs/index.php/tacl/article/view/570}{Improving
  distributional similarity with lessons learned from word embeddings}.
\newblock {\em Transactions of the Association for Computational Linguistics\/}
  3:211--225.
\newblock
  \href{https://tacl2013.cs.columbia.edu/ojs/index.php/tacl/article/view/570}{https://tacl2013.cs.columbia.edu/ojs/index.php/tacl/article/view/570}.

\bibitem[{Mikolov(2013)}]{Mikolov}
Tomas Mikolov. 2013.
\newblock word2vec.
\newblock https://code.google.com/archive/p/word2vec.

\bibitem[{Mikolov et~al.(2013{\natexlab{a}})Mikolov, Chen, Corrado, and
  Dean}]{Mikolov13b}
Tomas Mikolov, Kai Chen, Greg Corrado, and Jeffrey Dean. 2013{\natexlab{a}}.
\newblock Efficient estimation of word representations in vector space.
\newblock In {\em Workshop at ICLR\/}.

\bibitem[{Mikolov et~al.(2013{\natexlab{b}})Mikolov, Sutskever, Chen, Corrado,
  and Dean}]{Mikolov13c}
Tomas Mikolov, Ilya Sutskever, Kai Chen, Greg~S Corrado, and Jeff Dean.
  2013{\natexlab{b}}.
\newblock Distributed representations of words and phrases and their
  compositionality.
\newblock In {\em Advances in NIPS\/}. pages 3111--3119.

\bibitem[{Mikolov et~al.(2013{\natexlab{c}})Mikolov, Yih, and
  Zweig}]{Mikolov13a}
Tomas Mikolov, Wen-Tau Yih, and Geoffrey Zweig. 2013{\natexlab{c}}.
\newblock \href{http://www.aclweb.org/anthology/N13-1090}{Linguistic
  regularities in continuous space word representations}.
\newblock In {\em Proceedings of NAACL\/}. pages 746--751.
\newblock
  \href{http://www.aclweb.org/anthology/N13-1090}{http://www.aclweb.org/anthology/N13-1090}.

\bibitem[{Misra and Gries(1982)}]{Misra82}
Jayadev Misra and David Gries. 1982.
\newblock Finding repeated elements.
\newblock {\em Science of Computer Programming\/} 2(2):143--152.

\bibitem[{Napoles et~al.(2012)Napoles, Gormley, and Durme}]{Napoles12}
Courtney Napoles, Matthew Gormley, and Benjamin~Van Durme. 2012.
\newblock Annotated english gigaword ldc2012t21.

\bibitem[{Peng et~al.(2017)Peng, Li, Song, and Liu}]{Peng17}
Hao Peng, Jianxin Li, Yangqiu Song, and Yaopeng Liu. 2017.
\newblock Incrementally learning the hierarchical softmax function for neural
  language models.
\newblock In {\em Proceedings of AAAI (to appear)\/}.

\bibitem[{Pennington et~al.(2014)Pennington, Socher, and
  Manning}]{Pennington14}
Jeffrey Pennington, Richard Socher, and Christopher Manning. 2014.
\newblock \href{http://www.aclweb.org/anthology/D14-1162}{Glove: Global vectors
  for word representation}.
\newblock In {\em Proceedings of EMNLP\/}. pages 1532--1543.
\newblock
  \href{http://www.aclweb.org/anthology/D14-1162}{http://www.aclweb.org/anthology/D14-1162}.

\bibitem[{{\v R}eh{\r u}{\v r}ek and Sojka(2010)}]{Rehurek10}
Radim {\v R}eh{\r u}{\v r}ek and Petr Sojka. 2010.
\newblock Software framework for topic modelling with large corpora.
\newblock In {\em {Proceedings of the LREC 2010 Workshop on New Challenges for
  NLP Frameworks}\/}. pages 45--50.

\bibitem[{Turney and Pantel(2010)}]{Turney10}
Peter~D. Turney and Patrick Pantel. 2010.
\newblock From frequency to meaning: Vector space models of semantics.
\newblock {\em Journal of Artificial Intelligence Research\/} 37:141--188.

\bibitem[{Vitter(1985)}]{Vitter85}
Jeffrey~S. Vitter. 1985.
\newblock Random sampling with a reservoir.
\newblock {\em ACM Transactions on Mathematical Software\/} 11:37--57.

\bibitem[{Yin et~al.(2015)Yin, Schnabel, and Sch\"{u}tze}]{Yin15}
Wenpeng Yin, Tobias Schnabel, and Hinrich Sch\"{u}tze. 2015.
\newblock \href{http://aclweb.org/anthology/D15-1155}{Online updating of word
  representations for part-of-speech tagging}.
\newblock In {\em Proceedings of EMNLP\/}. pages 1329--1334.
\newblock
  \href{http://aclweb.org/anthology/D15-1155}{http://aclweb.org/anthology/D15-1155}.

\end{thebibliography}

\onecolumn
\appendix
\section{Note on Adaptive Unigram Table}

Algorithm \ref{alg:adaptive} illustrates the efficient implementation of the adaptive unigram table (\textit{c.f.}, Section 3.2.2). In line 8 and 10, $F$ and $\frac{\tau F}{z}$ are not always integers and therefore they are probabilistically converted into integers as explained in the paper.

Time complexity of Algorithm \ref{alg:adaptive} is $\mathcal{O}(1)$ per update in case of $\alpha=1.0$. When $|T|<\tau$, the update (line 8) takes $\mathcal{O}(1)$ time since we always have $F=1$. When $\tau\leq|T|$, we have $\tau\leq z$ and consequently $\frac{\tau F}{z}\leq 1$. This means that the update (line 10--13) takes $\mathcal{O}(1)$ time.

Even if $\alpha\not=1.0$, the value of $z$ becomes sufficiently large in practice, and thus the update becomes efficient as demonstrated in the experiment.

\begin{algorithm}[h]
\caption{Adaptive unigram table.}
\label{alg:adaptive}
\begin{algorithmic}[1]
\STATE $f(w)\leftarrow 0$ for all $w\in\vocab$
\STATE $z\leftarrow 0$
\FOR{$i=1,\dots,n$}
\STATE $f(w_i)\leftarrow f(w_i)+1$
\STATE $F\leftarrow f(w_i)^{\alpha}-(f(w_i)-1)^{\alpha}$
\STATE $z\leftarrow z+F$
\IF{$|T|<\tau$}
\STATE add $F$ copies of $w_i$ to $T$
\ELSE
\FOR{$t=1,\dots,\frac{\tau F}{z}$}
\STATE $j$ is randomly drawn from $[1,|T|]$
\STATE $T[j]\leftarrow w_i$
\ENDFOR
\ENDIF
\ENDFOR
\end{algorithmic}
\end{algorithm}

\section{Complete Proofs}

This appendix provides complete proofs of Theorems 1, 3, and 5.

\subsection{Proof of Theorem 1} \label{appendix:proof1}

\begin{proof}
The first order moment of $\Delta\mathcal{L}(\theta)$ can be rewritten as
\begin{align*}
\mathbb{E}[\Delta\mathcal{L}(\theta)]
&=\mathbb{E}\biggl[\frac{2ck}{n}\sum_{w\in\vocab}\sum_{v\in\vocab}\sum_{i=1}^n\delta_{w_i,w}(q_i(v)-q_n(v))\psi_{w,v}^{-}\biggr]\\
&=\frac{2ck}{n}\sum_{w\in\vocab}\sum_{v\in\vocab}\sum_{i=1}^n\mathbb{E}[\delta_{w_i,w}(q_i(v)-q_n(v))\psi_{w,v}^{-}]\\
&=\frac{2ck}{n}\sum_{w\in\vocab}\sum_{v\in\vocab}\sum_{i=1}^n\mathbb{E}[\rv{X}{i,w}(\rv{Y}{i,v}-\rv{Y}{n,v})\psi_{w,v}^{-}]\\
&=\frac{2ck}{n}\sum_{w\in\vocab}\sum_{v\in\vocab}\sum_{i=1}^n\biggl(\mathbb{E}[\rv{X}{i,w}\rv{Y}{i,v}]-\mathbb{E}[\rv{X}{i,w}\rv{Y}{n,v}]\biggr)\psi_{w,v}^{-}.
\end{align*}
Here, for any $i$ and $j$ such that $i\leq j$, we have
\begin{align*}
\mathbb{E}[\rv{X}{i,w}\rv{Y}{j,v}]
	&=\mathbb{E}[\rv{X}{i,w}\frac{1}{j}\sum_{j'=1}^{j}\rv{X}{j',v}]
	=\frac{1}{j}\sum_{j'=1}^{j}\mathbb{E}[\rv{X}{i,w}\rv{X}{j',v}]\\
	&=\frac{1}{j}\sum_{j'=1}^{j}\biggl(\mathbb{E}[\rv{X}{i,w}]\mathbb{E}[\rv{X}{j',v}]+\mathbb{V}[\rv{X}{i,w},\rv{X}{j',v}]\biggr)\\
	&=\mu_w\mu_v +\frac{1}{j}\var{w,v}.
\end{align*}
Therefore, we have
\begin{align*}
\mathbb{E}[\Delta\mathcal{L}(\theta)]
&=\frac{2ck}{n}\sum_{w\in\vocab}\sum_{v\in\vocab}\sum_{i=1}^{n}\biggl(\mu_w\mu_v+\frac{1}{i}\var{w,v}-\mu_w\mu_v-\frac{1}{n}\var{w,v}\biggl)\psi_{w,v}^{-}\\
&=\frac{2ck(H_{n}-1)}{n}\sum_{w\in\vocab}\sum_{v\in\vocab}\var{w,v}\psi_{w,v}^{-}.
\end{align*}
\end{proof}

\subsection{Proof of Theorem 3} \label{appendix:proof3}

To prove Theorem 3, we begin by examining the upper- and lower-bounds of $\mathbb{E}[\rv{X}{i,w}\rv{Y}{j,v}\rv{Y}{k,v}]$ in the following Lemma, and then make use of the bounds to evaluate the order of the second order moment of $\Delta\mathcal{L}(\theta)$.
\begin{lemma}
\label{lemma:bound}
For any $j$ and $k$ such that $j\leq k$, we have
\begin{align*}
\mathbb{E}[\rv{X}{i,w}\rv{Y}{j,v}\rv{Y}{k,v}]&\leq\frac{(jk-2j-k+2)\mu_w\mu_v^2+2j+k-2}{jk},\\
\mathbb{E}[\rv{X}{i,w}\rv{Y}{j,v}\rv{Y}{k,v}]&\geq\frac{(jk-2j-k+2)\mu_w\mu_v^2}{jk}.
\end{align*}
\end{lemma}
\begin{proof}
We have
\begin{align*}
\mathbb{E}[\rv{X}{i,w}\rv{Y}{j,v}\rv{Y}{k,v}]
&=\mathbb{E}[\rv{X}{i,w}\biggl(\frac{1}{j}\sum_{l=1}^{j}\rv{X}{l,v}\biggr)\biggl(\frac{1}{k}\sum_{m=1}^{k}\rv{X}{m,v}\biggr)]\\
&=\sum_{l=1}^{j}\sum_{m=1}^{k}\frac{\mathbb{E}[\rv{X}{i,w}\rv{X}{l,v}\rv{X}{m,v}]}{jk}.
\end{align*}
To prove the lemma, we rewrite the expression by splitting the set of $(l,m)$ into two subsets. Let $\mathcal{S}_{i}^{(j,k)}$ $(j\leq k)$ be a set of $(l,m)$ such that $\rv{X}{i,w}$, $\rv{X}{l,v}$, and $\rv{X}{m,v}$ are independent from each other (\textit{i.e.}, $i$, $l$, and $m$ are all different), and let $\bar{\mathcal{S}}_i^{(j,k)}$ be its complementary set:
\begin{align*}
\mathcal{S}_{i}^{(j,k)}&=\{(l,m)\in\{1,2,\dots,j\}\times\{1,2,\dots,k\}\mid i\not=l\land l\not=m \land m\not=i\},\\
\bar{\mathcal{S}}_{i}^{(j,k)}&=\{1,2,\dots,j\}\times\{1,2,\dots,k\}\setminus\mathcal{S}_{i}^{(j,k)}.
\end{align*}
Then, $\mathbb{E}[\rv{X}{i,w}\rv{Y}{j,v}\rv{Y}{k,v}]$ is upper-bounded as
\begin{align*}
\mathbb{E}[\rv{X}{i,w}\rv{Y}{j,v}\rv{Y}{k,v}]
&=\sum_{(l,m)\in\mathcal{S}_{i}^{(j,k)}}\frac{\mathbb{E}[\rv{X}{i,w}]\mathbb{E}[\rv{X}{l,v}]\mathbb{E}[\rv{X}{m,v}]}{jk}+
\sum_{(l,m)\in\bar{\mathcal{S}}_{i}^{(j,k)}}\frac{\mathbb{E}[\rv{X}{i,w}\rv{X}{l,v}\rv{X}{m,v}]}{jk}\\
&\leq\sum_{(l,m)\in\mathcal{S}_{i}^{(j,k)}}\frac{\mu_w\mu_v^2}{jk}+\sum_{(l,m)\in\bar{\mathcal{S}}_{i}^{(j,k)}}\frac{1}{jk}\\
&=\frac{|\mathcal{S}_i^{(j,k)}|\mu_w\mu_v^2 + |\bar{\mathcal{S}}_i^{(j,k)}|}{jk},
\end{align*}
where the inequality holds because $\rv{X}{i,w}$, $\rv{X}{l,v}$, and $\rv{X}{m,v}$ are binary random variables and thus $\mathbb{E}[\rv{X}{i,w}\rv{X}{l,v}\rv{X}{m,v}]\leq 1$. Here, we have $|\bar{\mathcal{S}}_i^{(j,k)}|=2j+k-2$, because $\bar{\mathcal{S}}_{i
}^{(j,k)}$ includes $j$ elements such that $l=m$ and also includes $k-1$ and $j-1$ elements such that $i=l\not=m$ and $i=m\not=l$, respectively. And we consequently have $|\mathcal{S}_i^{(j,k)}|=jk-|\bar{\mathcal{S}}_i^{(j,k)}|=jk-2j-k+2$. Therefore, the upper-bound can be rewritten as
\begin{align*}
\mathbb{E}[\rv{X}{i,w}\rv{Y}{j,v}\rv{Y}{k,v}]&\leq\frac{(jk-2j-k+2)\mu_w\mu_v^2+2j+k-2}{jk}.
\end{align*}

Similarly, by making use of $0\leq\mathbb{E}[\rv{X}{i,w}\rv{X}{l,v}\rv{X}{m,v}]$, the lower-bound can be derived:
\begin{align*}
\mathbb{E}[\rv{X}{i,w}\rv{Y}{j,v}\rv{Y}{k,v}]
&=\sum_{(l,m)\in\mathcal{S}_{i}^{(j,k)}}\frac{\mathbb{E}[\rv{X}{i,w}]\mathbb{E}[\rv{X}{l,v}]\mathbb{E}[\rv{X}{m,v}]}{jk}+\sum_{(l,m)\in\bar{\mathcal{S}}_{i}^{(j,k)}}\frac{\mathbb{E}[\rv{X}{i,w}\rv{X}{l,v}\rv{X}{m,v}]}{jk}\\
&\geq\sum_{(l,m)\in\mathcal{S}_{i}^{(j,k)}}\frac{\mu_w\mu_v^2}{jk}+\sum_{(l,m)\in\bar{\mathcal{S}}_{i}^{(j,k)}}\frac{0}{jk}\\
&=\frac{|\mathcal{S}_{i}^{(j,k)}|\mu_w\mu_v^2}{jk}=\frac{(jk-2j-k+2)\mu_w\mu_v^2}{jk}.
\end{align*}
\end{proof}

Making use the above Lemma, we can prove Theorem 3.
\begin{proof}
 The upper-bound of $\mathbb{E}[\Delta\mathcal{L}(\theta)^2]$ is examined to prove the theorem. Let $\Psi_{i,n,w,v}=\delta_{w_i,w}(q_i(v)-q_n(v))\psi_{w,v}^{-}$. Making use of Jensen's inequality, we have
\begin{align*}
	\mathbb{E}[\Delta\mathcal{L}(\theta)^2]
	&=\mathbb{E}\biggl[\frac{4c^2k^2}{n^2}\biggl(\sum_{w\in\mathcal{W}}\sum_{v\in\mathcal{W}}\sum_{i=1}^{n}\Psi_{i,n,w,v}\biggr)^2\biggr]\\
	&=\mathbb{E}\biggl[\frac{4c^2k^2}{n^2}\vocabsize^4n^2\biggl(\sum_{w\in\mathcal{W}}\sum_{v\in\mathcal{W}}\sum_{i=1}^{n}\frac{1}{\vocabsize^2n}\Psi_{i,n,w,v}\biggr)^2\biggr]\\
	&\leq\mathbb{E}\biggl[\frac{4c^2k^2}{n^2}\vocabsize^4n^2\sum_{w\in\vocab}\sum_{v\in\vocab}\sum_{i=1}^{n}\frac{1}{\vocabsize^2n}\Psi_{i,n,w,v}\biggr]\\
	&=\frac{4c^2k^2\vocabsize^2}{n}\sum_{w\in\vocab}\sum_{v\in\vocab}\sum_{i=1}^{n}\mathbb{E}[\Psi_{i,n,w,v}^2].
\end{align*}
Furthermore, the term $\mathbb{E}[\Psi_{i,n,w,v}^2]$ is upper-bounded as
\begin{align*}
\mathbb{E}[\Psi_{i,n,w,v}^2]
&=\mathbb{E}[\delta_{w_i,v}^2(q_i(v)-q_n(v))^2(\psi_{w,v}^{-})^{2}]\\
&=\mathbb{E}[\delta_{w_i,v}(q_i(v)-q_n(v))^2(\psi_{w,v}^{-})^2]\\
&=\mathbb{E}[\rv{X}{i,w}(\rv{Y}{i,v}-\rv{Y}{n,v})^2](\psi_{w,v}^{-})^2\\
&=(\mathbb{E}[\rv{X}{i,w}\rv{Y}{i,v}^2]-2\mathbb{E}[\rv{X}{i,w}\rv{Y}{i,v}\rv{Y}{n,v}]+\mathbb{E}[\rv{X}{i,w}\rv{Y}{n,v}^2])(\psi_{w,v}^{-})^2\\
&\leq\biggl\{\frac{1}{i^2}\biggl((i^2-3i+2)\mu_w\mu_v^2+3i-2\biggr)\\
&\quad-2\frac{1}{in}(in-2i-n+2)\mu_w\mu_v^2\\
&\quad+\frac{1}{n^2}\biggl((n^2-3n+2)\mu_w\mu_v^2+3n-2\biggr)\biggl\}(\psi_{w,v}^{-})^2\\
&=\biggl\{(2\mu_w\mu_v^2-2)\frac{1}{i^2}+(-\mu_w\mu_v^2-\frac{4}{n}\mu_w\mu_v^2+3)\frac{1}{i}\\
&\quad+(2\mu_w\mu_v^2-2)\frac{1}{n^2}+(\mu_w\mu_v^2+3)\frac{1}{n}\biggl\}(\psi_{w,v}^{-})^2,
\end{align*}
where the above Lemma is used to derive the inequality. Therefore, we have
\begin{align*}
\sum_{i=1}^n\mathbb{E}[\Psi_{i,n,w,v}^2]
&\leq\sum_{i=1}^n\biggl\{(2\mu_w\mu_v^2-2)\frac{1}{i^2}+(-\mu_w\mu_v^2-\frac{4}{n}\mu_w\mu_v^2+3)\frac{1}{i}\\
&\quad+(2\mu_w\mu_v^2-2)\frac{1}{n^2}+(\mu_w\mu_v^2+3)\frac{1}{n}\biggr\}(\psi_{w,v}^{-})^2\\
&=\biggl\{(2\mu_w\mu_v^2-2)H_{n,2}+(-\mu_w\mu_v^2-\frac{4}{n}\mu_w\mu_v^2+3)H_{n}\\
&\quad+(2\mu_w\mu_v^2-2)\frac{1}{n}+(\mu_w\mu_v^2+3)\biggl\}(\psi_{w,v}^{-})^2,
\end{align*}
where $H_{n,2}$ represents the generalized harmonic number of order $n$ of 2. Since $H_{n,2}\leq H_{n}=\mathcal{O}(\log(n))$, we have $\sum_{i=1}^n\mathbb{E}[\Psi_{i,n,w,v}^2]=\mathcal{O}(\log(n))$ and consequently $\mathbb{E}[\Delta\mathcal{L}(\theta)^2]=\mathcal{O}(\frac{\log(n)}{n})$.
\end{proof}

\subsection{Proof of Theorem 5}

\begin{proof}
The proof is made by the squeeze theorem. Let $\l=\mathcal{L}_\text{B}(\hat{\theta})-\mathcal{L}_\text{B}(\theta^*)$. Then, Chebyshev's inequality gives, for any $\epsilon_1>0$, 
\begin{align*}
\lim_{n\rightarrow\infty}\frac{\mathbb{V}[\l]}{\epsilon_1^2}
&\geq\lim_{n\rightarrow\infty}\Pr\biggl[|\l-\mathbb{E}[\l]|\geq\epsilon_1\biggr]\\
&=\lim_{n\rightarrow\infty}\Pr\biggl[\l-\mathbb{E}[\l]\leq-\epsilon_1\biggl]
+\Pr\biggl[\epsilon_1\leq\l-\mathbb{E}[\l]\biggr]\\
&=\lim_{n\rightarrow\infty}\Pr\biggl[\l\leq\mathbb{E}[\l]-\epsilon_1\biggl]
+\Pr\biggl[\mathbb{E}[\l]+\epsilon_1\leq\l\biggr].
\end{align*}
Remind that Eq.~(2) in Lemma 4 means that for any $\epsilon_2>0$, there exists $n'$ such that if $n'\leq n$ then $|\mathbb{E}[\l]|<\epsilon_2$. Therefore we have
\begin{align*}
\lim_{n\rightarrow\infty}\frac{\mathbb{V}[\l]}{\epsilon_1^2}
&\geq\lim_{n\rightarrow\infty}\Pr\biggl[\l\leq\mathbb{E}[\l]-\epsilon_1\biggl]
+\Pr\biggl[\mathbb{E}[\l]+\epsilon_1\leq\l\biggr]\\
&\geq\lim_{n\rightarrow\infty}\Pr\biggl[\l\leq-\epsilon_2-\epsilon_1\biggl]
+\Pr\biggl[\epsilon_2+\epsilon_1\leq\l\biggr]\\
&=\lim_{n\rightarrow\infty}\Pr\biggl[|\l|\geq\epsilon_1+\epsilon_2\biggr]\geq 0.
\end{align*}
The arbitrary property of $\epsilon_1$ and $\epsilon_2$ allows $\epsilon_1+\epsilon_2$ to be rewritten as $\epsilon$. Also, Eq.~(3) in Lemma 4 implies that $\lim_{n\rightarrow\infty}\frac{\mathbb{V}[\l]}{\epsilon_1^2}=0$. Therefore, the squeeze theorem gives the proof.
\end{proof}

\section{Theoretical Analysis in Smoothed Case} \label{appendix:smoothed_case}

This appendix investigates the convergence of the first and second order moment of $\Delta\mathcal{L}(\theta)$ in the smoothed case.

\subsection{Convergence of the first order moment of $\Delta\mathcal{L}(\theta)$}

The first order moment of $\Delta\mathcal{L}(\theta)$ in the smoothed case is given as
\begin{align*}
\mathbb{E}[\Delta\mathcal{L}(\theta)]
&=\frac{2ck}{n}\sum_{w\in\vocab}\sum_{v\in\vocab}\sum_{i=1}^{n}\biggl(\mathbb{E}[\rv{X}{i,w}\rv{Z}{i,v}]-\mathbb{E}[\rv{X}{i,w}\rv{Z}{n,v}]\biggr)\psi_{w,v}^{-}.\label{eq:moment1}
\end{align*}
Let us investigate $\mathbb{E}[\rv{X}{i,w}\rv{Z}{j,v}]$ as we did  $\mathbb{E}[\rv{X}{i,w}\rv{Y}{j,v}]$ in the unsmoothed case. Let $\phi_w=g_w(\mu)-\sum_{v\in\vocab}M_{w,v}g_v(\mu)$. Then, for any $i$ and $j$ such that $i\leq j$, we have 
\begin{align*}
\mathbb{E}[\rv{X}{i,w}\rv{Z}{j,v}]
&\approx\mathbb{E}[\rv{X}{i,w}\biggl(g_v(\mu)+\sum_{v'\in\vocab}M_{v,v'}(\rv{Y}{j,v'}-g_{v'}(\mu))\biggr)]\\
&=\mathbb{E}[\rv{X}{i,w}(\sum_{v'\in\vocab}M_{v,v'}\rv{Y}{j,v'}+\phi_v)]\\
&=\sum_{v'\in\vocab}M_{v,v'}E[\rv{X}{i,w}\rv{Y}{j,v'}]+\phi_v E[\rv{X}{i,w}]\\
&=\sum_{v'\in\vocab}M_{v,v'}(\mu_w\mu_{v'}+\frac{1}{j}\var{w,v'})+\mu_w\phi_v\\
&=\sum_{v'\in\vocab}M_{v,v'}\mu_w\mu_{v'}+\mu_w\phi_v+\frac{1}{j}\sum_{v'\in\vocab}M_{v,v'}\var{w,v'}.
\end{align*}
Therefore, plugging the above equation into $\mathbb{E}[\Delta\mathcal{L}(\theta)]$ yields $\mathbb{E}[\Delta\mathcal{L}(\theta)]\approx\mathcal{O}(\frac{\log(n)}{n})$.

\subsection{Convergence of the second order moment of $\Delta\mathcal{L}(\theta)$}

Next, let us examine the convergence of the second order moment of $\Delta\mathcal{L}(\theta)$. This can be confirmed by inspecting $\mathbb{E}[\rv{X}{i,w}\rv{Z}{j,v}\rv{Z}{k,v}]$ and then $\mathbb{E}[\Psi_{i,n,w,v}^2]$ analogously to the unsmoothed case.

For any $i$, $j$, and $k$ such that $i\leq j\leq k$, we have
\begin{align*}
\mathbb{E}[\rv{X}{i,w}\rv{Z}{j,v}\rv{Z}{k,v}]
&\approx\mathbb{E}[\rv{X}{i,w}\biggl(\sum_{v'\in\vocab}M_{v,v'}\rv{Y}{j,v'}+\phi_v\biggr)\biggl(\sum_{v''\in\vocab}M_{v,v''}\rv{Y}{k,v''}+\phi_v\biggr)]\\
&=\sum_{v'\in\vocab}\sum_{v''\in\vocab}M_{v,v'}M_{v,v''}\mathbb{E}[\rv{X}{i,w}\rv{Y}{j,v'}\rv{Y}{k,v''}]\\
&\quad+\sum_{v'\in\vocab}M_{v,v'}\phi_v\mathbb{E}[\rv{X}{i,w}\rv{Y}{j,v'}]+\sum_{v''\in\vocab}M_{v,v''}\phi_v\mathbb{E}[\rv{X}{i,w}\rv{Y}{k,v''}]+\phi_v^2\mathbb{E}[\rv{X}{i,w}]\\
&=\sum_{v'\in\vocab}\sum_{v''\in\vocab}M_{v,v'}M_{v,v''}\mathbb{E}[\rv{X}{i,w}\rv{Y}{j,v'}\rv{Y}{k,v''}]\\
&\quad+\sum_{v'\in\vocab}M_{v,v'}\phi_v(\mu_w\mu_{v'}+\frac{1}{j}\var{w,v'})\\
&\quad+\sum_{v''\in\vocab}M_{v,v''}\phi_v(\mu_w\mu_{v''}+\frac{1}{k}\Sigma_{w,v''})
+\mu_w\phi_v^2.
\end{align*}
Therefore, we have
\begin{align*}
\mathbb{E}[\Psi_{i,n,w,v}^2]
&=\mathbb{E}[\rv{X}{i,w}(\rv{Z}{i,v}-\rv{Z}{n,v})^2]\psi_{w,v}^2\\
&\approx\sum_{v'\in\vocab}\sum_{v''\in\vocab}M_{v,v'}M_{v,v''}\biggl(\mathbb{E}[\rv{X}{i,w}\rv{Y}{i,v'}\rv{Y}{i,v''}]\\
&\quad-2\mathbb{E}[\rv{X}{i,w}\rv{Y}{i,v'}\rv{Y}{n,v''}]+\mathbb{E}[\rv{X}{i,w}\rv{Y}{n,v'}\rv{Y}{n,v''}]\biggl)\psi_{w,v}^2.
\end{align*}
Using similar bounds to Lemma 3, we also have $\sum_{i=1}^{n}\mathbb{E}[\Psi_{i,n,w,v}^2]\approx\mathcal{O}(\log(n))$ and consequently $\mathbb{E}[\Delta\mathcal{L}(\theta)^2]\approx\mathcal{O}(\frac{\log(n)}{n})$. 

\section{Theoretical Analysis of Mini-batch SGNS}

This appendix demonstrates that Theorems 2 and 3 also hold for the mini-batch SGNS, that is, the first and second order moments of $\Delta\mathcal{L}(\theta)$ are in the order of $\mathcal{O}(\frac{\log(n)}{n})$. We here investigate the mini-batch setting in which $M$ words, as opposed to a single word in the case of incremental SGNS, are processed at a time.

\begin{definition}
Let $\rv{Y}{i,w}^{(M)}$ be a random variable that represents $q_i(w)$ when $\alpha=1.0$ and the mini-batch size is $M$. Then, it is given as
\begin{align*}
\rv{Y}{i,w}^{(M)} &= \rv{Y}{b(i,M),w}
\end{align*}
where $b(i,M)=\lceil\frac{i}{M}\rceil\times M$. Note that we always have $\rv{Y}{n,w}^{(M)}=\rv{Y}{n,w}$ and $i\leq b(i,M)$.
\end{definition}

We first examine the first order moment of $\Delta\mathcal{L}(\theta)$ by taking a similar step as the proof of Theorem 1. The first order moment of $\Delta\mathcal{L}(\theta)$ is given as
\begin{align*}
\mathbb{E}[\Delta\mathcal{L}(\theta)]&=\frac{2ck}{n}\sum_{w\in\vocab}\sum_{v\in\vocab}\sum_{i=1}^{n}\biggl(\mathbb{E}[\rv{X}{i,w}\rv{Y}{j,v}^{(M)}]-\mathbb{E}[\rv{X}{i,w}\rv{Y}{n,v}^{(M)}]\biggr)\psi_{w,v}^{-}\\
&=\frac{2ck}{n}\sum_{w\in\vocab}\sum_{v\in\vocab}\sum_{i=1}^{n}\biggl(\mathbb{E}[\rv{X}{i,w}\rv{Y}{j,v}^{(M)}]-\mathbb{E}[\rv{X}{i,w}\rv{Y}{n,v}]\biggr)\psi_{w,v}^{-}\\
&=\frac{2ck}{n}\sum_{w\in\vocab}\sum_{v\in\vocab}\rho_{w,v}\psi_{w,v}^{-}\biggl(\sum_{i=1}^{n}\frac{1}{b(i,M)}-\sum_{i=1}^{n}\frac{1}{n}\biggr).
\end{align*}
Because we have
\begin{align}
\sum_{i=1}^{n}\frac{1}{b(i,M)}&\leq\sum_{i=1}^{n}\frac{1}{i}=H_{n}=\mathcal{O}(\log(n)),
\end{align}
we have $\mathbb{E}[\Delta\mathcal{L}(\theta)]=\mathcal{O}(\frac{\log(n)}{n})$.

Next, we investigate the second order moment of $\mathbb{E}[\Delta\mathcal{L}(\theta)]$. Analogously to the last inequality of the proof of Theorem 3, we have
\begin{align*}
\sum_{i=1}^n\mathbb{E}[\Psi_{i,n,w,v}^2]
&\leq\sum_{i=1}^n\biggl\{(2\mu_w\mu_v^2-2)\frac{1}{b(i,M)^2}+(-\mu_w\mu_v^2-\frac{4}{n}\mu_w\mu_v^2+3)\frac{1}{b(i,M)}\\
&\quad+(2\mu_w\mu_v^2-2)\frac{1}{n^2}+(\mu_w\mu_v^2+3)\frac{1}{n}\biggr\}(\psi_{w,v}^{-})^2.
\end{align*}
Since we have
\begin{align}
\sum_{i=1}^{n}\frac{1}{b(i,M)^2}&\leq\sum_{i=1}^{n}\frac{1}{i^2}=H_{n,2}=\mathcal{O}(\log(n)), 
\end{align}
it can be proven that $\mathbb{E}[\Delta\mathcal{L}(\theta)^2]=\mathcal{O}(\frac{\log(n)}{n})$.

\section{Experimental Configurations}

This appendix details the experimental configurations that are not described in the paper.

\subsection{Verification of theorems}

The vocabulary set in the Gigaword corpus was reduced to $1000$ by converting infrequent words into the same special tokens because it is expensive to evaluate the expectation terms in $\Delta\mathcal{L}(\theta)$ for a large vocabulary set. 

The parameter $\theta$ was set to $100$-dimensional vectors each element of which is drawn from $[-0.5, 0.5]$ uniformly at random. In preliminary experiments we confirmed that the result is not sensitive to the choice of the parameter value. Note that the same parameter value was used for all $n$. We set $c$ and $k$ as $c=5$ and $k=5$.

The mean $\mu_w$ and covariances $\var{w,v}$ are required to compute the theoretical value of the first order moment. They were given as the maximum likelihood estimations from the entire Gigaword corpus.

\subsection{Quality of word embeddings}

Table~\ref{tab:config} summarizes the training configurations. Those parameter values were used for both incremental and batch SGNS. The learning rate was set to $0.1$ for \textbf{incremental} and \textbf{batch}, which use AdaGrad to adjust the learning rate. On the other hand, the learning rate of \textbf{w2v}, which uses linear decay function to adjust the learning rate, was set as the default value of $0.025$. 

In the word similarity and the analogy tasks, we use $\tvec{w}+\cvec{w}$ as an embedding of the word $w$ \cite{Pennington14,Levy15}. The analogy task was performed by using 3CosMul \cite{Levy15}.

\begin{table}[t]
\centering
\begin{tabular}{cc}\hline
\textbf{Parameter} & \textbf{Value} \\\hline
Embedding size & $400$ \\
Number negative samples & $10$ \\
Subsampling threshold & $1.0\times10^{-5}$ \\
Subsampling method & dirty \\
Window size & $10$ \\
Smoothing parameter $\alpha$ & $0.75$ \\ \hline
\end{tabular}
\caption{Training configurations. Incremental SGNS used the incrementally-updated frequency for the subsampling.}
\label{tab:config}
\end{table}

\subsection{Update time}

The experiment was conducted on Intel\textsuperscript{\textregistered{}} Xeon\textsuperscript{\textregistered{}} 2GHz CPU. The update time was averaged over five trials.

\end{document}